\documentclass[sigconf]{acmart}

\usepackage{verbatim}%多行注释
\usepackage{tabulary}
\usepackage{array}
\usepackage{graphicx}
\usepackage{booktabs}
\usepackage{multirow}
\usepackage{multicol}
\usepackage{booktabs}
\usepackage{subfigure}
\usepackage{caption}
\usepackage[ruled,linesnumbered]{algorithm2e}
\usepackage{graphicx}

\AtBeginDocument{%
  \providecommand\BibTeX{{%
    \normalfont B\kern-0.5em{\scshape i\kern-0.25em b}\kern-0.8em\TeX}}}

\copyrightyear{2021}
\acmYear{2021}
\setcopyright{acmcopyright}\acmConference[MM '21]{Proceedings of the 29th ACM International Conference on Multimedia}{October 20--24, 2021}{Virtual Event, China}
\acmBooktitle{Proceedings of the 29th ACM International Conference on Multimedia (MM '21), October 20--24, 2021, Virtual Event, China}
\acmPrice{15.00}
\acmDOI{10.1145/3474085.3475378}
\acmISBN{978-1-4503-8651-7/21/10}

\newcommand\blfootnote[1]{% 
\begingroup 
\renewcommand\thefootnote{}\footnote{#1}% 
\addtocounter{footnote}{-1}% 
\endgroup 
}

\begin{document}
\fancyhead{}

\title{Towards Adversarial Patch Analysis and Certified Defense against Crowd Counting}

\author{Qiming Wu$^\dagger$}
\affiliation{%
  \institution{School of Electronic Information and Communications, Huazhong University of Science and Technology}
  \city{Wuhan}
  \country{China}
}
\email{qimingwu@hust.edu.cn}

\author{Zhikang Zou$^\dagger$}
\affiliation{%
  \institution{Department of Computer Vision Technology (VIS), Baidu Inc.}
  \city{Shenzhen}
  \country{China}
}
\email{zouzhikang@baidu.com}

\author{Pan Zhou$^*$}
\affiliation{
\institution{The Hubei Engineering Research
Center on Big Data Security, School
of Cyber Science and Engineering, Huazhong University of Science and Technology}
  \city{Wuhan}
  \country{China}
}
\email{panzhou@hust.edu.cn}

\author{Xiaoqing Ye}
\affiliation{%
  \institution{Department of Computer Vision Technology (VIS), Baidu Inc.}
  \city{Shanghai}
  \country{China}
}
\email{yexiaoqing@baidu.com}

\author{Binghui Wang}
\affiliation{%
  \institution{Department of Computer Science, Illinois Institute of Technology}
  \city{Chicago}
  \country{USA}}
\email{bwang70@iit.edu}

\author{Ang Li}
\affiliation{%
  \institution{Department of Electrical and Computer Engineering, Duke University}
  \city{Durham}
  \country{USA}}
\email{ang.li630@duke.edu}

\renewcommand{\shortauthors}{Wu and Zou, et al.}

\begin{abstract}
    Crowd counting has drawn much attention due to its importance in safety-critical surveillance systems. Especially, deep neural network (DNN) methods have significantly reduced estimation errors for crowd counting missions. Recent studies have demonstrated that DNNs are vulnerable to adversarial attacks, i.e., normal images with human-imperceptible perturbations could  mislead DNNs to make false predictions. In this work, we propose a robust attack strategy called Adversarial Patch Attack with Momentum (APAM) to systematically evaluate the robustness of crowd counting models, where the attacker's goal is to create an adversarial perturbation that severely degrades their performances, thus leading to public safety accidents (e.g., stampede accidents). Especially, the proposed attack leverages the extreme-density background information of input images to generate robust adversarial patches via a series of transformations (e.g., interpolation, rotation, etc.). We observe that by perturbing less than 6\% of image pixels, our attacks severely degrade the performance of crowd counting systems, both digitally and physically. To better enhance the adversarial robustness of crowd counting models, we propose the first regression model-based Randomized Ablation (RA), which is more sufficient than Adversarial Training (ADT) (Mean Absolute Error of RA is 5 lower than ADT on clean samples and 30 lower than ADT on adversarial examples). Extensive experiments on five crowd counting models demonstrate the effectiveness and generality of the proposed method.
\end{abstract}




\maketitle

\blfootnote{$^\dagger$Equal Contribution.}
\blfootnote{$^*$Corresponding author: Pan Zhou.}

\begin{figure}
   \includegraphics[width=0.48\textwidth]{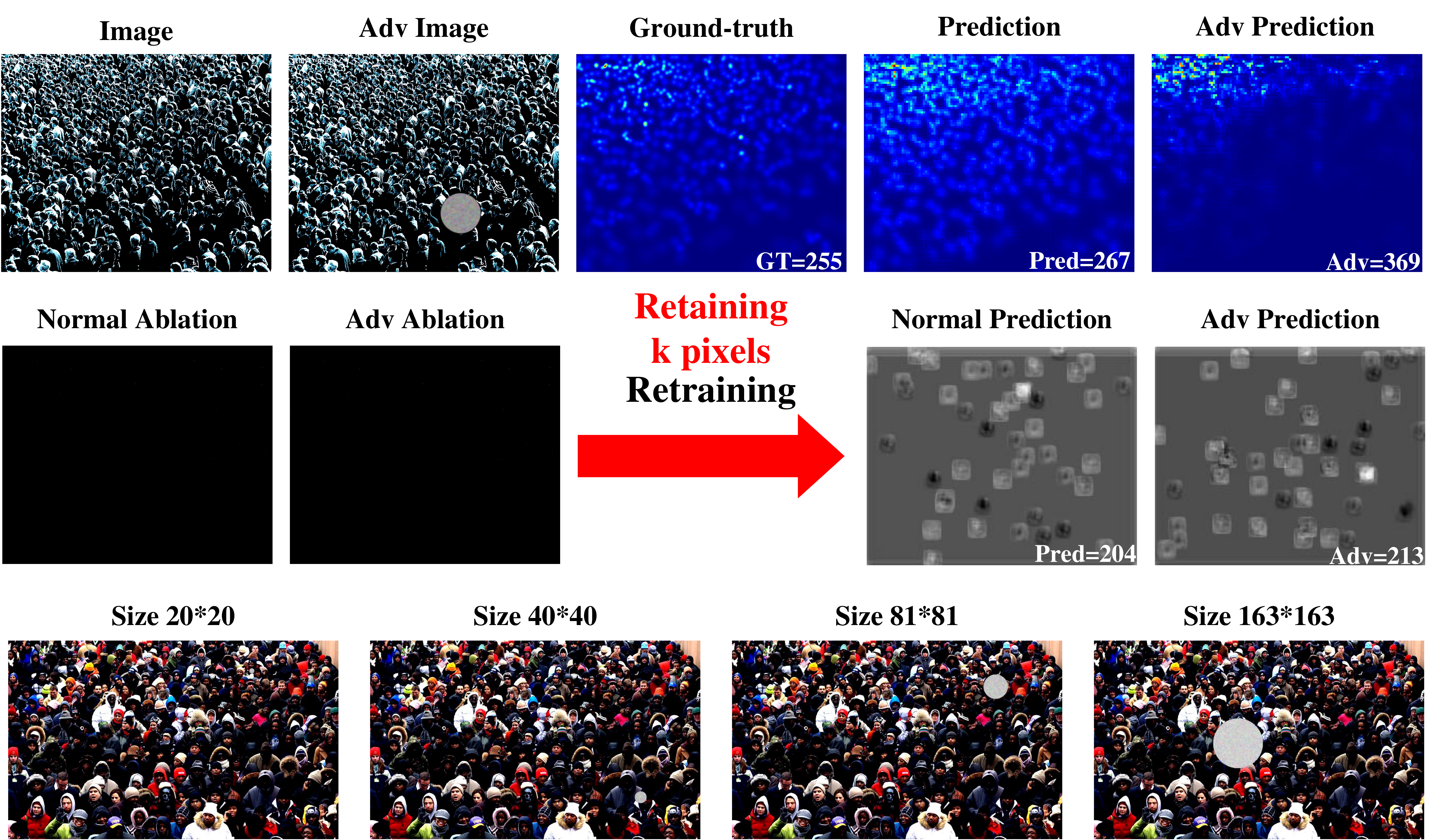}
   \caption{The overview of the proposed adversarial attack methods and certified defense strategy. The first row shows adversarial examples generated by APAM algorithm (retaining parameter $k = 45$) and the second row shows that after retraining with ablated images, the output density maps of normal and adversarial images have little differences. The third row illustrates the imperceptibility of the generated adversarial patch. It is hard for human eyes to find the patch in a congested scenes when patch size is $20 \times 20$ or $40\times40$.}
   %and defense strategies against crowd counting systems on ShanghaiTech dataset \cite{zhang2016single}.} %When adding adversarial patches to the input images, the output density map  has been perturbed thus influencing counting the number of crowd based on the image.}
   \label{fig:1}
   \vspace{-5mm}
\end{figure}

\section{Introduction}
Due to the global outbreak of the COVID-19 virus, a large amount of public places are requiring people to keep the social distance. Therefore, the video surveillance systems composed of crowd counting models \cite{zhang2015crowd} are undoubtedly gaining prominence in the administration. As crowd counting is one of the most significant applications of deep neural networks (DNNs) and has been adopted in many safety-critical scenarios like video surveillance \cite{zhang2015crowd} and traffic control \cite{guerrero2015extremely}. However, recent works have demonstrated the vulnerability of DNNs towards adversarial examples \cite{szegedy2013intriguing,goodfellow2014explaining}, this definitely provides attackers with new interface to perform attacks for malicious purposes. Naturally, attackers may intend to generate perturbations that fool DNN models to count the crowd inaccurately so as to increase the possibility of causing public safety accidents (e.g., viral infection, stampede and severe traffic accidents.).

To achieve the adversarial goal, we design the Adversarial Patch Attack with Momentum (APAM) algorithm on the crowd counting systems. Compared with the classic adversarial examples restricted with $L_0$, $L_2$ or $L_{\infty}$ distance metrics \cite{moosavi2016deepfool, goodfellow2014explaining, carlini2017towards, papernot2016limitations}, the APAM attacks are not limited to the norm bound (since the $L_p$ norm bound is to guarantee the imperceptibility of adversarial examples, we claim that the adversarial patch generated by APAM algorithm can achieve the goal by reducing the patch size, e.g., In the third row of Fig. \ref{fig:1}, $20\times20$ and $40 \times 40$ patches are hard to find in the dense crowd scenes) and our proposed algorithm can accelerate the optimization with momentum to obtain robust adversarial patches. Therefore, performing adversarial patch attacks to evaluate the robustness of crowd counting systems is definitely a promising research direction, which is the major focus of this work.\par

\textbf{Concurrent Work.} There is only one recent work \cite{liu2019using} which introduces a defense strategy to study the robustness of crowd counting systems. However, the proposed defense strategy relies on the depth information of RGBD datasets, which is not generally available in the crowd counting applications. Additionally, they ignore the importance of background information in crowd counting, where such information 
is effectively exploited by our attacks. \par

\textbf{Motivation.} We observe that the extremely dense background has been the key obstacle in crowd counting. Numerous works are dedicated to overcome this challenge by utilizing deeper convolutional neural networks (CNN). However, due to the inherent vulnerability of CNN, such CNN-based methods leave a new vulnerable interface for potential attackers. This motivates us to design novel adversarial attacks so as to better understand and improve the robustness of the CNN-based crowd counting system. \par

\textbf{APAM Attacks.} Our proposed attacks aim to leverage the congested background information for generating the adversarial patches. In addition, we enhance the algorithm of generating the adversarial patches with momentum and remove the norm bound to strengthen the attack capabilities (Fig. \ref{fig:1} shows the generated adversarial examples on a typical dataset).\par

\textbf{Certified Defense via Randomized Ablation.} We propose a certified defense strategy against APAM attacks on crowd counting, namely, randomized ablation. Our defense strategy consists of two parts: image ablation and certificate retraining crowd counting models. The first step is inspired by the recent advance in image classifier certification \cite{levine2020robustness}.
Specifically, randomized ablation is effective against APAM attacks because the ablation results of normal image $x$ and adversarially perturbed image $\tilde{x}$ are likely to be same (e.g., retaining 45 pixels for each images in Fig. \ref{fig:1}). Note that several other methods have been proposed to certify the robustness, like dual approach \cite{dvijotham2018dual},
interval analysis \cite{gowal2018effectiveness}, and abstract interpretations \cite{mirman2018differentiable}. Compared with these methods, randomized ablation is simpler and more importantly, scalable to complicated models.\par

Our major contributions are summarized as follows:
\begin {itemize}
\item[$\bullet$] To the best of our knowledge, this is the first work to propose a systematic and practical method on the evaluation of the robustness of crowd counting models via adversarial patch attacks and the certified defense strategy (i.e., randomized ablation).

\item[$\bullet$] We design a robust adversarial patch attack framework called Adversarial Patch Attack with Momentum (APAM) to create effective adversarial perturbations on mainstream CNN-based crowd counting models. 

\item[$\bullet$]We implement the APAM attack algorithm on the network in two forms: white-box attack and black-box attack. We evaluate the proposed attacks in both digital and physical spaces. Qualitative and quantitative results demonstrate that our attacks significantly degrade the performances of models, hence, pose severe threats to crowd counting systems.

\item[$\bullet$] We provide the first theoretical guarantee of the adversarial robustness of crowd counting models via randomized ablation. More practically, after training the verification models with this strategy, we achieve the significant robustness enhancement. Meanwhile, our proposed method defeats the traditional adversarial (patch) training both on clean sample and adversarial example evaluation tests.
\end {itemize}

\section{Related Work}

\subsection{Background of Crowd Counting}
Crowd analysis is an inter-disciplinary research topic with researchers from different domains \cite{sindagi2018survey}, the approaches of studying crowd counting are also characterized by multidisciplinary integration \cite{idrees2013multi, chen2013cumulative}. Initial research approaches are divided into three categories \cite{sindagi2018survey}: detection-based methods, regression-based methods and density estimation-based methods. A survey of classical crowd counting approaches 
is available in \cite{loy2013crowd}. However, the quality of predicted density map generated by these classical crowd counting methods are limited when applied in congested scenes. Thanks to the success of CNNs in other fields \cite{chen2015deepdriving, wang2014fingerprint, tzortzis2007deep}, researchers recently propose the CNN-based density estimation approaches \cite{li2018csrnet, zou2018net, zhang2016single, liu2019context, sindagi2017cnn} to find a way out of the dilemma. A survey of the CNN-based crowd counting methods is available in \cite{sindagi2018survey}.\par

Although more and more impressive models have yielded exciting results \cite{Acemap.226409936,Acemap.30997604,Acemap.230230506,Acemap.191633431,Acemap.354247512,Acemap.248673452,Acemap.203671797} on the bench-mark datasets, their robustness has not been reasonably understood. In particular, we focus on the systematic evaluation and understanding of the robustness of these five models \cite{liu2019context, li2018csrnet, sindagi2017cnn, zou2018net, zhang2016single} in this article.\par

\subsection{Adversarial Attacks}
\textbf{$L_p$ Norm Bounded Adversarial Perturbation.}
Recent works have demonstrated the existence of adversarial examples in deep neural networks \cite{szegedy2013intriguing}, a variety of methods such as FGSM \cite{goodfellow2014explaining}, Deepfool \cite{moosavi2016deepfool}, C\&W \cite{carlini2017towards} and JSMA \cite{papernot2016limitations}, have been proposed to generating adversarial examples bounded by $L_p$ norms. Basically, the problem is formulated as: $ \vert \vert \tilde{x} - x \vert \vert_p \le \epsilon$, where $\epsilon$ is the parameter control the strength of perturbation. Researchers often choose $L_0$, $L_2$ and $L_{\infty}$ metrics in practice. $L_0$ norm counts the number of changed pixels in $\tilde{x}$, $L_2$ norm is formulated as $\vert \vert \tilde{x} - x \vert \vert_2 = (\Delta x_1)^2 + (\Delta x_2)^2 + \cdots + (\Delta x_n)^2$ and $L_{\infty}$ norm is formulated as $\vert \vert \tilde{x} - x \vert \vert_{\infty} = max \{ \Delta x_1, \Delta x_2,\cdots, \Delta x_n \}$. The attacker aims to find the optimal adversarial example $\tilde{x} = x + \delta$ to fool the neural networks. A survey of adversarial examples in deep learning is available in \cite{yuan2019adversarial}.

\textbf{Empirical Defense Strategy.} Many defense strategies \cite{akhtar2018threat} have been proposed, such as  network distillation \cite{papernot2016distillation},adversarial training \cite{goodfellow2014explaining, huang2015learning}, adversarial detecting \cite{lin2017detecting,lu2017safetynet,metzen2017detecting}, input reconstruction \cite{gu2014towards}, classifier robustifying \cite{bradshaw2017adversarial}, network verification \cite{katz2017towards} and ensemble defenses \cite{meng2017magnet}, etc. However, these defense strategies have a common major flaw: almost all of the above defenses have an effect on only part of the adversarial attacks, and even have no defense effect on some invisible and powerful attacks.\par

\textbf{Adversarial Patch Attacks.} As numerous empirical defense strategies have been proposed to defend against $L_p$ norm adversarial attacks, researchers explore adversarial patch attacks to further fool DNNs \cite{Brown2017AdversarialP}. Attackers obtain the adversarial patch via optimizing the traditional equation  $\widehat{p}=\mathop{\arg\max}\limits_{p} \{ log Pr(\widehat{y}) \vert \widehat{A}(x, l, p, t) \}$ in \cite{Brown2017AdversarialP}. Specifically, $\hat{A}(x, l, p, t)$ is a patch application operator, where $x$ is the input image, $l$ is the patch location, $p$ is the patch, and $t$ is the image transformation. Because the adversarial patch attacks are image-independent, it allows attackers to launch attacks easily regardless of scenes and victim models. Moreover, the state-of-the-art empirical defenses, which focus on small perturbations, may not be effective against large adversarial patches. Meanwhile, the adversarial patch attack has been widely applied to many safety-sensitive applications, such as face recognition \cite{yang2019design, pautov2019adversarial} and object detection \cite{saha2019adversarial,lee2019physical,liu2018dpatch}, which inspires this work either.

\section{Problem Setup}
We first define the regression models, and then introduce the background of crowd counting. After that, we define our attacks against crowd counting models.

\textbf{Definition 3.1} (\emph{Regression Models}). In statistical modeling, regression models refers to models that can estimate the relationships between a dependent variable and independent variables.\par

\subsection{Regression based Crowd Analysis}
Given a set of $N$ labeled images $D=\{(x_i, l_i) \}_{i=1}^N$, where 
$x_i \in \mathbb{R}^{H_{I} \cdot W_{I} \cdot C_{I}}$ and $H_I$, $W_I$, and $C_I$ are the height, width, and channel number of the image, respectively. $l_i$ is the $i-th$ ground truth density map of image $x_i$. Then, a crowd counting system aims to learn a model $f_{\theta}$, parameterized by $\theta$, by using these labeled images and solving the following optimization problem:
\begin{equation}
   \min \frac{1}{2N} \sum^N_{i=1} \vert \vert f_\theta (x_i) - l_i \vert \vert^2.
\end{equation}
Note that researchers recently have adopted more effective loss functions in crowd counting \cite{cheng2019learning,ma2019bayesian} and we consider the most commonly used $L_2$ loss function. Moreover, different crowd counting models \cite{li2018csrnet, zou2018net, zhang2016single, liu2019context, sindagi2017cnn} will use different architectures. For instances, MCNN \cite{zhang2016single} uses Multi-column convolutional neural networks to predict the density map. The learned model $f_{\theta}$ can be used to predict the crowd count in a testing image $x$. Specifically, $f_{\theta}$ takes $x$ as an input and outputs the predicted density map $f_{\theta}(x_i)$. Then, the crowd count in $x$ is estimated by summing up all values of the density map.\par

As crowd counting systems are of great importance in safety-critical applications, such as video surveillance, an adversary is motivated to fool the systems to count the crowd inaccurately so as to increase the possibility of causing public safety accidents. Next, we will introduce the threat model and formally define our problem.\par

\subsection{Threat Model}

{\bf Adversary's knowledge:} Depending on how much information an adversary knows about the crowd counting system, we characterize an adversary's knowledge in terms of the following two aspects:
\begin{itemize}
\item \textbf{Full knowledge.} In this scenario, the adversary is assumed to have all knowledge of the targeted crowd counting system, for example, model parameters, model architecture, etc. \par

\item \textbf{Limited knowledge.} In this scenario, the adversary has no access to the model parameters of the targeted crowd counting system. 
In practice, however, there exist various crowd counting systems different from the targeted system.  
We assume the adversary can adopt these crowd counting systems as the substitution and perform an attack on these substitute systems. \par

\end{itemize}

\noindent {\bf Adversary's capability.}
We consider different capabilities for an adversary to launch attacks. 
In the full knowledge setting, an adversary can launch the \emph{white-box attack}, i.e., it can generate an adversarial patch to a testing image by directly leveraging the model parameters of the targeted crowd counting system. In the limited knowledge setting, an adversary can launch the \emph{black-box attack}, i.e., an adversary cannot leverage the model information of the targeted system, but can use several substitute crowd counting systems to generate an adversarial patch. Plus, we also study \emph{physical attack}. In this scenario, we aim to attack crowd counting systems in a real-world case. Attacking crowd counting systems in the real world is much more difficult than in the digital space. With less information obtained, the attacker directly poses the generated adversarial patch in the scene to fool networks. This requires the adversarial patch to generalize well across various crowd counting systems. To this end, we randomly select a shopping mall to evaluate our generated adversarial patch.

\noindent {\bf Adversary's goal.} 
Given a set of testing images with ground truth crowd counts and a targeted crowd counting system, an adversary aims to find an adversarial patch for each testing image such that the perturbed image has a predicted crowd count by the targeted system that largely deviates from the ground truth.\par

\subsection{Problem Definition}

Given a crowd counting model $f_{\theta}$ and a testing image $x$ with ground truth density map $l$, we aim to add an adversarial perturbation (i.e., adversarial patch in our work) $\delta$ to the testing image $x$ such that the model $f_{\theta}$ predicts the crowd count in the testing image as the adversary desires. 
Note that the crowd count is calculated by the summation of all values of the density map, modifying the prediction of the crowd count equals to modifying the density map. Suppose the adversary aims to learn a targeted density map $l^\ast$ with a perturbation $\delta$. 
Then, our attack can be defined as follows:
\begin{equation}
    \mathop{\min}_{\delta} H(x+\delta, x),\quad s.t. \quad  f_\theta(x+\delta)=l^\ast, \label{equ:2}
\end{equation}
where $H$ is a distance function. 

Directly solving Eq. \ref{equ:2} is challenging due to that the equality constraint involves a highly nonlinear model $f_\theta$. An alternative way is to put the constraint into the objective function. Specifically, 
\begin{equation}
    %\mathbb{E}_{(x,l)\in D} 
    \mathop{\arg\min}_{\delta :\  \vert \vert \delta \vert \vert_q \leq \varepsilon} J(f_\theta(x+\delta);l^\ast), \label{equ:3}
\end{equation}
where $\varepsilon$ is the budget constraint and $J$ is a loss function (e.g., cross-entropy loss).\par

% 解释P_mask, P_patch
\begin{figure}[t]
\centering
 \includegraphics[width=0.48\textwidth]{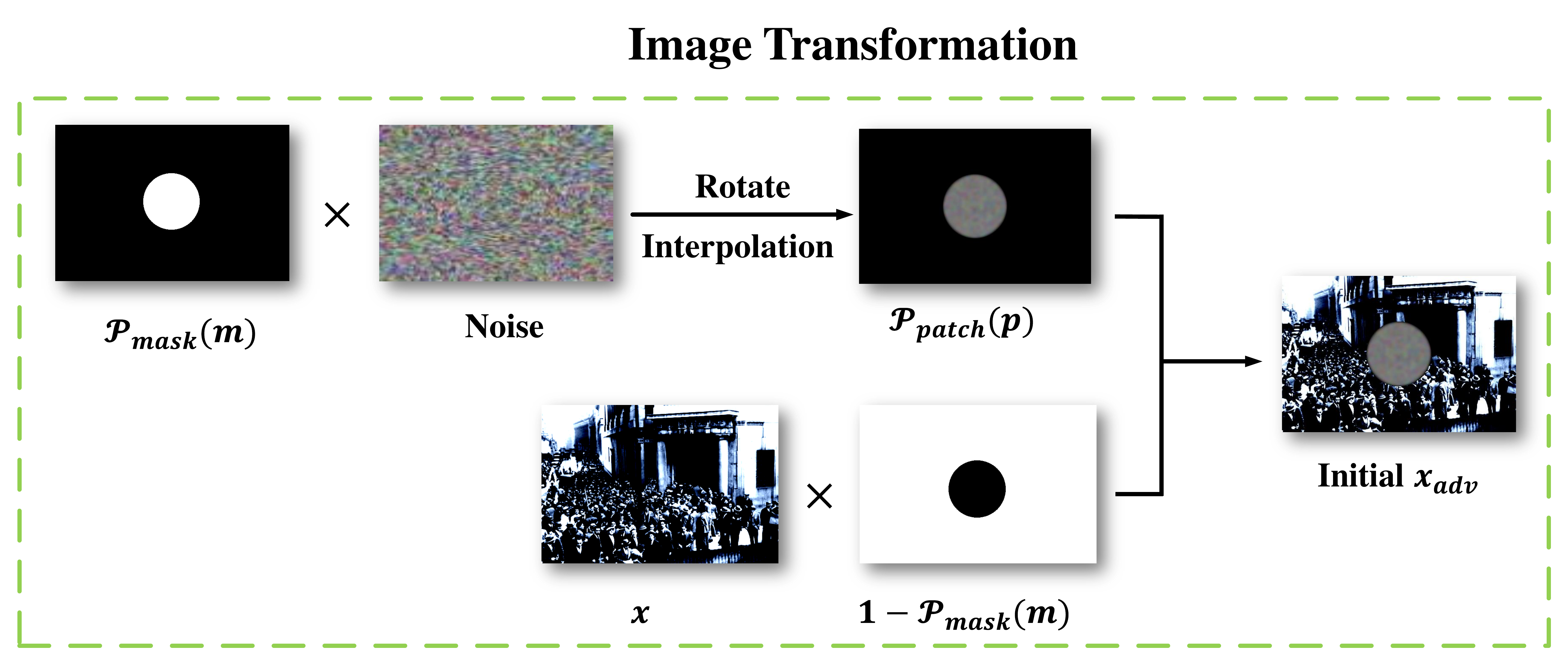}
 \caption{Description of the Image Transformation Function. Black pixels have value 0 and white pixels have value 1.
 %The black pixel value is zero and the white pixel value is one.
 }\label{fig:6}
 \vspace{-3mm}
\end{figure}

\begin{figure*}
   \includegraphics[width=0.8\textwidth]{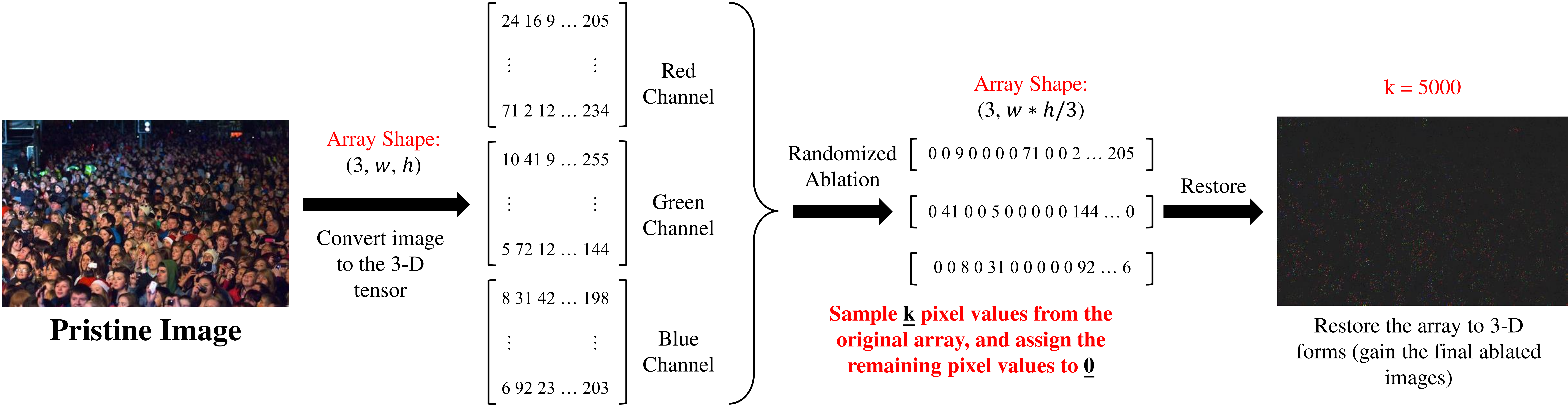}
   \caption{The illustration of the proposed defense method. The details of the randomized ablation are summarized in Section. \ref{section:RA_defense}. This image process is done during the certificate training, and we could gain the specific ablated image when given the pristine one. Moreover, the model will not overfit on the obtained ablated image dataset since the ablated images vary in training epochs (i.e., the results of sampling $k$ pixels of the image vary from epochs to epochs).}
   \label{fig:randomized ablation}
   %\vspace{-4mm}
\end{figure*}

\section{The Proposed Attack}
In this section, we introduce the proposed APAM attack in following orders: we first design it under the white-box setting. And then, the black-box settings.\par

\subsection{White-box Attack}
Our proposed white-box attack consists of two phases: adversarial patch initialization and adversarial patch optimization with momentum.\par

\subsubsection{Adversarial Patch Initialization.} Our patch initialization process includes two steps: image transformation and interpolation smoothness.\par

\textbf{Image Transformation.} As researchers have demonstrated that Cyber-physical systems can destroy perturbations crafted using digital-only algorithms \cite{lu2017no} and the physical perturbation can be affected by environmental factors, including viewpoints \cite{eykholt2018robust}. To solve the problem, we manipulate the image transformation function to make the patch more robust through Eq. \ref{equ:7} and Eq. \ref{equ:19}. Also, the pipelines of image transformation function is vividly illustrated in Fig. \ref{fig:6}.\par

Inspired by \cite{qiu2019semanticadv}, we propose to interpolate the generated adversarial patch with tensor $\beta$. As $\beta \in \mathbb{R}^{H_I \times W_I \times C_I}$ is one tensor in the image space and $\beta_{h,w,c} \in (0,1)$, where $h \in [1, H_I]$, $w \in [1,W_I]$ and $c \in [1, C_I]$. Manipulating this kind of interpolation guarantees the quality of the adversarial patch. We then formulate the initialization of adversarial examples as follows:
\begin{equation}
    \tilde{x}= \mathcal{P}_{patch}(p) + (1- \mathcal{P}_{mask}(m)) \cdot x,
    \label{equ:7}
\end{equation}
where
\begin{equation}
    \mathcal{P}_{patch} = \mathcal{I} \circ \mathcal{R} \circ ( \mathcal{P}_{mask}(m) \cdot Noise ).  \label{equ:19}
\end{equation}

\textbf{Interpolation Smoothness.} As the interpolation tensor $\beta$ contains many parameters, we simplify the problem by considering setting a smoothness constraint on $\beta$. As defined in Eq. \ref{equ:8}, the smoothness loss is widely used in processing images as pixel-wise de-noising objectives \cite{mansimov2015generating, johnson2016perceptual}.\par

\begin{equation}
\begin{aligned}
    \mathcal{L}_{sm}(\beta) & = \sum_{w=1}^{W_{I}-1} \sum_{h=1}^{H_{I}} \vert \vert \beta_{w+1,h}-\beta_{w,h} \vert \vert^2 \\
    & + \sum_{w=1}^{W_{I}} \sum_{h=1}^{H_{I}-1} \vert \vert \beta_{w,h+1} - \beta_{w,h} \vert \vert^2,
\end{aligned} \label{equ:8}
\end{equation}

\subsubsection{Adversarial Optimization Objectives}

\

\textbf{Adversarial Patch Generation.} We acquire the final adversarial patch $x_t^{adv}$ through minimizing the objective $\mathcal{L}_{APAM}$ in Eq. \ref{equ:5} with the initial adversarial example $\tilde{x}$ generated through the image transformation function $\mathcal{T}$. Particularly, our objective function $\mathcal{L}_{APAM}$ has two parts: the adversarial loss $J$ and the smoothness regularization $\mathcal{L}_{sm}(\beta)$. The smoothness term aims to smooth the optimization of adversarial patches and guarantee the perceptual quality of the patch since it is scaled and rotated during image transformation process. To make the final generated adversarial robust, we propose to minimize the following objective function:

\begin{equation}
\begin{aligned}
    x_t^{adv} & =\mathop{\arg\min}_{\tilde{x}} \mathcal{L}_{APAM} (\tilde{x}, l^\ast, f_\theta) \\ 
    & = J (f_\theta(\tilde{x}); l^\ast) + \gamma \cdot \mathcal{L}_{sm}(\beta),
    \label{equ:5}
\end{aligned}
\end{equation}
where $\gamma$ is a hyperparameter to balance the two terms.

\textbf{Momentum-based Optimization.} Momentum method is usually used in the gradient descent algorithm to accelerate the optimization process with the help of the memorization of previous gradients. As the adversary searches the optimal adversarial example $\tilde{x}$ in the high-dimensional space, there are high possibilities of being trapped in small humps, narrow valleys and poor local minima or maxima \cite{duch1998optimization, dong2018boosting}. In order to break the dilemma, we integrate 
% the 
momentum into the optimization of the adversarial patch so as to update more stably and further enhance the potential capability of the attacker. Specifically, 
\begin{equation}
    q_{t+1} = \mu q_t + \frac{\nabla_{x}J_{\theta}(\tilde{x}_t,l^\ast)}{\vert \vert \nabla_{x}J_{\theta}(\tilde{x}_t,l^\ast) \vert \vert}.
    \label{equ:16}
\end{equation}

Researchers propose to boost the traditional adversarial attacks with momentum in Eq. \ref{equ:16} to generate perturbations $\tilde{x}_{t+1} = \tilde{x}_t + \epsilon \cdot sign(q_{t+1}) $  iteratively \cite{dong2018boosting}. Inspired by their work, we extend the optimization process of the adversarial patch with momentum. By adding variables to control the exponentially weighted average, the optimization process can be smoothed and accelerated.

Therefore, the generated adversarial patch is capable of transferring across various models and its attack ability simultaneously remains strong, thus demonstrating robust adversarial perturbations.\par

\subsection{Black-box Attack}
In a black-box attack, an adversary has no access to the internal structure of victim models. However, the adversary can adopt substitute models to generate an adversarial patch.
In order to make the adversarial patch robust, we consider jointly attacking multiple crowd counting systems. 
Specifically, the objective function of our black-box attack is defined as follows:
\begin{equation}
    \mathop{\arg\min}_{\tilde{x}} \frac{1}{n} 
    \sum_{i=1}^{n} \mathcal{L}_{APAM}((\tilde{x}, l^\ast, f_i),
    \label{equ:9}
\end{equation}
where each $f_i$ denotes a substitute crowd counting model with parameter $\theta_i$.\par

\section{Certified Defense via Randomized Ablation}\label{section:RA_defense}

In this section, we will clarify the principles of randomized ablation. A vivid process is depicted in Fig. \ref{fig:randomized ablation}. We use $S$ to denote the set of all possible pixel values and $\mathcal{X} = S^d$ represents the set of all images and $k$ is the retention constant. When using the null symbol NULL during encoding images, we adopt the mean pixel encoding method proposed in \cite{levine2020robustness} for simplicity and efficiency. Similarly, $[d]=\{1,...,d\}$ is the set of indices and $\mathcal{H}(d,k)\subseteq{\mathcal{P}([d])}$ denotes all sets of $k$ unique indices. Define random variable $\mathcal{T} \sim \mathcal{U}(d,k)$, where $\mathcal{U}(d,k)$ is the uniform distribution over $\mathcal{H}(d,k)$.\par
Then, we extend the randomized ablation scheme from classifier setting ($\mathbb{R}^n \rightarrow \{0,1\}^n$) to the real-valued function settings ($\mathbb{R}^n \rightarrow [0,1]^n$). 
% For achieving 
To achieve this goal, we introduce a top-K overlap metric $R(x,\tilde{x},K)$ proposed in \cite{ ghorbani2019interpretation} to measure the adversarial robustness of crowd counting models.\par

\newtheorem{thm}{Theorem}
\begin{thm}\label{theorem1}
For the images $x$, $\tilde{x}$ and ablated images $\mathcal{A}(x) = ABLATE(x,\mathcal{T})$, base crowd counting model with parameter $f_{\theta}: R^n \xrightarrow{} [0,1]^n$, with patch size $n$ given:
\begin{equation}
    \frac{C_n^k}{C_d^k} \le Pr(R(f_{\theta}(\mathcal{A}(x)),f_{\theta}(\mathcal{A}(\tilde{x})),K) = i) \le \frac{C_{d-n}^k}{C_d^k},
    \label{equ:theroem}
\end{equation}
where $i \in [0,k]$ denotes the number of overlapping largest top k elements of output density maps of $\mathcal{A}(x)$ and $\mathcal{A}(\tilde{x})$. Intuitively, we can derive the upper bound and lower bound as:
\begin{equation}
    \underline{Pr} = Pr(R(f_{\theta}(\mathcal{A}(x)),f_{\theta}(\mathcal{A}(\tilde{x})),K) = 0) = \frac{C_n^k}{C_d^k},
\end{equation}

\begin{equation}
    \overline{Pr} = Pr(R(f_{\theta}(\mathcal{A}(x)),f_{\theta}(\mathcal{A}(\tilde{x})),K) = k) = \frac{C_{d-n}^k}{C_d^k},
\end{equation}
where $k$ is the number of retained pixels in the image.
\end{thm}

\begin{proof}
We now follow the notations to complete the proof. Note that the $i$ in Eq. \ref{equ:theroem} counts the number of top-k largest pixels in two output density maps (normal ablated and adversarial ablated images). Next, we will prove the upper and the lower possibility to bound $Pr$ with probability at least $\alpha \in [\underline{Pr},\ \overline{Pr}]$.\par
If $i = k$, then we have $\mathcal{T} \cap (x\ominus \tilde{x}) = \emptyset$, which indicates $x$ and $\tilde{x}$ are identical at all indices in $\mathcal{T}$ (That is, the ablated results of $x$ and $\tilde{x}$ are identical, $ABLATE(x,\mathcal{T}) = ABLATE(\tilde{x},\mathcal{T})$). In this case, we have:
\begin{align}
    Pr(\mathcal{T} \cap (x\ominus \tilde{x}) = \emptyset) = \frac{{{d-\vert x\ominus \tilde{x} \vert} \choose k}}{{{d} \choose {k}}} \le \frac{C_{d-n}^k}{C_d^k} = \overline{Pr},
\end{align}
where ${d \choose k}$ and $C_d^{k}$ represent total ways of uniform choices of $k$ elements from $d$. 

If $i \le k$, we have $\mathcal{T}\cap(x\ominus \tilde{x}) \neq \emptyset$. Then, similarly:

\begin{align}
    & Pr(\mathcal{T} \cap (x\ominus \tilde{x}) \neq \emptyset) = 1 - Pr(\mathcal{T} \cap (x\ominus \tilde{x}) = \emptyset) \\ & = 1 - \frac{{{d-\vert x\ominus \tilde{x} \vert} \choose k}}{{{d} \choose {k}}} = 1 - \frac{C_{d-n}^k}{C_d^k} \ge \frac{C_n^k}{C_d^k} = \underline{Pr},
\end{align}\label{proof:worst case}where the final inequality denotes the \underline{\textbf{worst case}}: the retained pixels in the adversarial image $\tilde{x}$ are all picked from the adversarial perturbation regions of the image.
\end{proof}

\begin{table*}[htb]
\centering
\setlength{\tabcolsep}{1.4mm}{
\begin{tabular}{|c|c|c|c|c|c|c|c|c|c|c|}
\hline
\multirow{2}{*}{Models} & \multicolumn{10}{c|}{\textbf{White-box Attacks}}                                                                                                                                                 \\ \cline{2-11} 
                        & \textbf{MAE(0)} & \textbf{RMSE(0)} & \textbf{MAE(20)} & \textbf{RMSE(20)} & \textbf{MAE(40)} & \textbf{RMSE(40)} & \textbf{MAE(81)} & \textbf{RMSE(81)} & \textbf{MAE(163)} & \textbf{RMSE(163)} \\ \hline
\textbf{CSRNet}         & 68.20           & 115.00           & 112.56           & 154.23            & 179.44           & 204.64            & 312.48           & 396.43            & 463.53            & 589.35             \\ \hline
\textbf{DA-Net}         & 71.60           & 104.90           & 102.42           & 149.12            & 141.26           & 227.84            & 146.97           & 231.57            & 250.45            & 334.47             \\ \hline
\textbf{MCNN}           & 110.20          & 173.20           & 229.54           & 275.78            & 252.14           & 304.12            & 317.71           & 378.52            & 432.05            & 495.31             \\ \hline
\textbf{CAN}            & 62.30           & 100.00           & 303.26           & 407.49            & 315.12           & 423.64            & 386.46           & 499.76            & 410.23            & 528.42             \\ \hline
\textbf{CMTL}           & 101.30          & 152.40           & 185.43           & 253.64            & 259.76           & 356.88            & 312.72           & 416.43            & 402.51            & 439.53             \\ \hline
\multirow{2}{*}{}       & \multicolumn{10}{c|}{\textbf{Black-box Attacks}}                                                                                                                                                 \\ \cline{2-11} 
                        & \textbf{MAE(0)} & \textbf{RMSE(0)} & \textbf{MAE(20)} & \textbf{RMSE(20)} & \textbf{MAE(40)} & \textbf{RMSE(40)} & \textbf{MAE(81)} & \textbf{RMSE(81)} & \textbf{MAE(163)} & \textbf{RMSE(163)} \\ \hline
\textbf{CSRNet}         & 68.20           & 115.00           & 100.45           & 134.17            & 130.43           & 203.52            & 270.86           & 348.92            & 425.06            & 489.77             \\ \hline
\textbf{DA-Net}         & 71.60           & 104.90           & 98.20            & 136.05            & 108.23           & 157.56            & 114.01           & 159.78            & 152.27            & 205.61             \\ \hline
\textbf{MCNN}           & 110.20          & 173.20           & 136.85           & 197.31            & 142.68           & 211.66            & 151.09           & 221.88            & 164.05            & 226.01             \\ \hline
\textbf{CAN}            & 62.30           & 100.00           & 193.73           & 245.48            & 234.46           & 284.72            & 320.50           & 362.03            & 405.03            & 478.32             \\ \hline
\textbf{CMTL}           & 101.30          & 152.40           & 107.03           & 163.40            & 110.24           & 187.65            & 136.60           & 205.52            & 170.02            & 224.76             \\ \hline
\end{tabular}
}\caption{Experimental results of the proposed APAM methods. We implement the white-box and black-box attacks. Note that "MAE(20)" denotes the MAE value of patch size $20\times20$ attacks. All of these experiments are done on the ShanghaiTech dataset part A \cite{zhang2016single}.}\label{tab: all attack results}
\vspace{-5mm}
\end{table*}

\section{Experiments}

% various attack targets
\begin{table}[htb]
\centering
\setlength{\tabcolsep}{1.5mm}{
\begin{tabular}{|c|c|c|c|c|c|c|}
\hline
\multicolumn{7}{|c|}{\textbf{White-box Attack Results with Various Attack Targets}}                                   \\ \hline
\textbf{Models} & \textbf{M(5)} & \textbf{R(5)} & \textbf{M(20)} & \textbf{R(20)} & \textbf{M(100)} & \textbf{R(100)} \\ \hline
\textbf{CSRNet} & 371.72        & 519.35        & 492.14         & 612.73         & 532.70          & 649.24          \\ \hline
\textbf{DA-Net} & 164.24        & 285.37        & 399.75         & 417.20         & 423.56          & 473.57          \\ \hline
\textbf{MCNN}   & 377.23        & 492.06        & 456.72         & 591.34         & 479.33          & 625.41          \\ \hline
\textbf{CAN}    & 363.83        & 502.62        & 436.64         & 649.38         & 487.91          & 674.63          \\ \hline
\textbf{CMTL}   & 309.84        & 412.94        & 463.22         & 598.76         & 487.57          & 637.88          \\ \hline
\end{tabular}
}
\caption{Testing APAM attacks with different targets (5GT, 20GT, 100GT). Intuitively, the second and third column denote the normal performances of victim models. In the table, $M$ and $R$ denote MAE and RMSE. We attack these networks under the white-box setting.}\label{tab:various attack target}
\vspace{-4mm}
\end{table}

% theiretical possibility
\begin{table}[htb]
\centering
\vspace{-3mm}
\setlength{\tabcolsep}{1.5mm}{
\begin{tabular}{@{}ccccl@{}}
\toprule
\textbf{Patch Size}       & \textbf{20$\times$20} & \textbf{40$\times$40} & \textbf{81$\times$81} & \multicolumn{1}{c}{\textbf{163$\times$163}} \\ \midrule
\textbf{$\overline{Pr}$}  & 0.9752                & 0.9043                & 0.6611                & 0.1827                                      \\
\textbf{$\underline{Pr}$} & 1.7265e-147           & 4.0071e-120           & 1.7745e-92            & 3.9698e-65                                  \\ \bottomrule
\end{tabular}
}
\caption{Theoretical possibility analysis of the four selected adversarial patch. Note that the Upper bound possibility $\overline{Pr}$ means the possibility of picking non-adversarial pixels from the image. The $\underline{Pr}$ denotes the probability of \emph{worst case} (depicted in Eq. \ref{proof:worst case}) happening and we find in fact the randomized ablation method can almost avoid it ($\underline{Pr}$ is always close to zero)}.
\label{tab:theoretical possibility}
\vspace{-6mm}
\end{table}

%certificate defense results
\begin{table*}[htb]
\centering
\setlength{\tabcolsep}{1mm}{
\begin{tabular}{@{}ccccccccccc@{}}
\toprule
\multicolumn{11}{c}{\textbf{Certificate Defense Results}}                                                                                                                                                                                                                                                                   \\ \midrule
                                     & \textbf{MAE(0)} & \multicolumn{1}{c|}{\textbf{RMSE(0)}} & \textbf{MAE(20)} & \multicolumn{1}{c|}{\textbf{RMSE(20)}} & \textbf{MAE(40)} & \multicolumn{1}{c|}{\textbf{RMSE(40)}} & \textbf{MAE(81)} & \multicolumn{1}{c|}{\textbf{RMSE(81)}} & \textbf{MAE(163)} & \textbf{RMSE(163)} \\ \midrule
\multicolumn{1}{c|}{\textbf{CSRNet}} & 75.28           & \multicolumn{1}{c|}{134.87}           & 100.54           & \multicolumn{1}{c|}{142.67}            & 137.82           & \multicolumn{1}{c|}{192.93}            & 247.33           & \multicolumn{1}{c|}{362.87}            & 341.53            & 410.97             \\
\multicolumn{1}{c|}{\textbf{DA-Net}} & 89.54           & \multicolumn{1}{c|}{125.46}           & 97.34            & \multicolumn{1}{c|}{132.33}            & 124.78           & \multicolumn{1}{c|}{189.75}            & 141.22           & \multicolumn{1}{c|}{223.89}            & 199.76            & 302.76             \\
\multicolumn{1}{c|}{\textbf{MCNN}}   & 117.32          & \multicolumn{1}{c|}{185.44}           & 152.78           & \multicolumn{1}{c|}{263.11}            & 190.27           & \multicolumn{1}{c|}{288.50}            & 253.32           & \multicolumn{1}{c|}{367.24}            & 374.55            & 458.32             \\
\multicolumn{1}{c|}{\textbf{CAN}}    & 74.24           & \multicolumn{1}{c|}{117.76}           & 274.57           & \multicolumn{1}{c|}{382.32}            & 296.84           & \multicolumn{1}{c|}{406.88}            & 346.76           & \multicolumn{1}{c|}{463.59}            & 381.64            & 503.16             \\
\multicolumn{1}{c|}{\textbf{CMTL}}   & 114.52          & \multicolumn{1}{c|}{174.29}           & 169.75           & \multicolumn{1}{c|}{226.57}            & 221.37           & \multicolumn{1}{c|}{312.42}            & 276.38           & \multicolumn{1}{c|}{383.25}            & 362.97            & 421.24             \\ \bottomrule
\end{tabular}
}
\caption{Certificate defense results of $k=45$. Note that the results are the defense mechanism against white-box APAM attacks.}\label{tab:defense table}
\vspace{-5mm}
\end{table*}

%defense ablation
\begin{table*}[htb]
\centering
\setlength{\tabcolsep}{1mm}{
\begin{tabular}{@{}ccccccccccc@{}}
\toprule
\multicolumn{11}{c}{\textbf{Defense Ablation Study}}                                                                                                                                                                                                                                                                                                                               \\ \midrule
                                                                                            & \textbf{MAE(0)} & \multicolumn{1}{c|}{\textbf{RMSE(0)}} & \textbf{MAE(20)} & \multicolumn{1}{c|}{\textbf{RMSE(20)}} & \textbf{MAE(40)} & \multicolumn{1}{c|}{\textbf{RMSE(40)}} & \textbf{MAE(81)} & \multicolumn{1}{c|}{\textbf{RMSE(81)}} & \textbf{MAE(163)} & \textbf{RMSE(163)} \\ \midrule
\multicolumn{1}{c|}{\textbf{\begin{tabular}[c]{@{}c@{}}Adversarial \\ Training\end{tabular}}} & {\underline{121.26}}    & \multicolumn{1}{c|}{{\underline{191.81}}}     & {\underline{188.32}}     & \multicolumn{1}{c|}{{\underline{271.56}}}      & {\underline{223.56}}     & \multicolumn{1}{c|}{{\underline{301.04}}}      & {\underline{282.37}}     & \multicolumn{1}{c|}{{\underline{372.75}}}      & {\underline{397.73}}      & {\underline{461.03}}       \\ \midrule
\multicolumn{1}{c|}{$k=45$}                                                                 & 117.32          & \multicolumn{1}{c|}{185.44}           & 152.78           & \multicolumn{1}{c|}{263.11}            & 190.27           & \multicolumn{1}{c|}{288.50}            & 253.32           & \multicolumn{1}{c|}{367.24}            & 374.55            & 458.32             \\
\multicolumn{1}{c|}{$k=100$}                                                                & 118.53          & \multicolumn{1}{c|}{188.97}           & 151.19           & \multicolumn{1}{c|}{264.65}            & 186.38           & \multicolumn{1}{c|}{291.44}            & 251.90           & \multicolumn{1}{c|}{370.18}            & 369.88            & 449.14             \\
\multicolumn{1}{c|}{$k=200$}                                                                & 127.68          & \multicolumn{1}{c|}{215.34}           & 159.45           & \multicolumn{1}{c|}{293.66}            & 201.15           & \multicolumn{1}{c|}{346.23}            & 262.07           & \multicolumn{1}{c|}{405.69}            & 383.52            & 470.75             \\ \bottomrule
\end{tabular}
}
\caption{Defense ablation study results. We compare the randomized ablation method with the most commonly used adversarial training (generating the adversarial patch during the model training) and all the experiments are done on the MCNN \cite{zhang2016single} model structure.}
\label{tab:ablation study results.}
\vspace{-5mm}
\end{table*}

% training defense models
\begin{comment}
\begin{table}[htb]
\centering
\setlength{\tabcolsep}{1.5mm}{
\begin{tabular}{|c|c|c|c|c|c|}
\hline
                                                                   & \textbf{MCNN} & \textbf{CSRNet} & \textbf{CAN} & \textbf{CMTL} & \textbf{DA-Net} \\ \hline
\textbf{\begin{tabular}[c]{@{}c@{}}Learning\\ Rate\end{tabular}}   & 1e-5          & 1e-5            & 1e-7         & 1e-5          & 1e-5            \\ \hline
\textbf{\begin{tabular}[c]{@{}c@{}}Training\\ Epochs\end{tabular}} & 800           & 800             & 1000         & 800           & 800             \\ \hline
\textbf{\begin{tabular}[c]{@{}c@{}}Batch\\ Size\end{tabular}}      & 1             & 1               & 1            & 1             & 1               \\ \hline
\textbf{Optimizer}                                                 & Adam          & Adam            & SGD          & Adam          & Adam            \\ \hline
\textbf{Momentum}                                                  & 0.9           & 0.9             & 0.95         & 0.9           & 0.9             \\ \hline
\end{tabular}
}
\caption{Training parameters in the randomized ablation method. Note that the majority of the training parameters are remain the same of their previous works.}\label{tab:defense training parameter}
\vspace{-7mm}
\end{table}
\end{comment}

\subsection{Experiment Setup}\label{sec:experiment_setup}
\textbf{Dataset and Networks.} We select five well pre-trained crowd counting models, which are publicly available: CSRNet \cite{li2018csrnet}, DA-Net \cite{zou2018net}, MCNN \cite{zhang2016single}, CAN \cite{liu2019context} and CMTL \cite{sindagi2017cnn} as our verification models thanks to their prevalent usage in the field. Additionally, we select the ShanghaiTech dataset \cite{zhang2016single} for retraining and evaluation because it is the most representative dataset in the field, which contains 1198 images with over 330,000 people annotated. Besides, we adopt mean absolute error and root mean squared error (defined in Eq. \ref{equ:mae&rmse}) as the evaluation metric.
\begin{equation}
\footnotesize
{
    MAE = \frac{1}{N} \cdot \sum_{i=1}^{N} \vert C^{GT}_i-\tilde{C}_i \vert,\ \ \ RMSE = \sqrt{\frac{1}{N} \cdot \sum_{i=1}^{N}(C^{GT}_i-\tilde{C}_i)^{2}},
    }\label{equ:mae&rmse}
\end{equation}
where $N$ is the number of images, $C^{GT}_i$ and $\tilde{C}_i$ separately denote the $i$-th ground-truth counting and the counting of the corresponding adversarial image. We apply the following settings in our implementations: we set the attack target as $l^{\ast} = 10\cdot GT$ and set $\gamma = 0.01$ in Eq. \ref{equ:5} to balance the two terms. Moreover, four patch sizes are selected in our experiments: $20 \times 20$, $40 \times 40$, $81 \times 81$ and $163 \times 163$ to represent 0.07\%, 0.31\%, 1.25\% and 5.06\% of the image size, respectively. Moreover, training details of randomized ablation are summarized in Supplementary Materials.\par

\subsection{APAM in the White-box Setting}
The upper part of Table \ref{tab: all attack results} shows results of white-box APAM attacks. We observe that the effectiveness of attacks strengthens as patch size climbs from 20 to 163. We find that all of the five crowd counting models suffer most when patch size reaches $163 \times 163$ since the adversarial perturbation reaches the maximum. However, the robustness varies from model to model. For example, compared with other models, DA-Net \cite{zou2018net} remains relatively robust against white-box APAM attacks. For MCNN \cite{zhang2016single} and CAN \cite{liu2019context}, their performances degrade significantly even when the adversarial patch is small. Moreover, we compute the relative increase percentage of these values and find that the vulnerability of networks depends on whether the specific patch size is reached or not. For instances, the fastest increase percentage of MAE and RMSE values of CMTL \cite{sindagi2017cnn} are 40.09\% and 40.70\% when the patch size changes from 20 to 40. Then, we infer the threshold value of CMTL \cite{sindagi2017cnn} is between 20 to 40. We find the effectiveness of APAM attacks is affected by the attack target $l^\ast$. In the experiment, we use three different attack target $l^{\ast} = 5GT, 20GT, 100GT$ to study the effectiveness of attack target $l^\ast$. From Table \ref{tab:various attack target}, we observe that the values of MAE and RMSE increase when the $l^\ast$ becomes larger and the performances of victim networks are basically in accord with those of the attack experiments with $l^{\ast} = 10GT$. \par

\subsection{APAM in the Black-box Setting}
Since the mainstream crowd counting models have three popular structures (dilated convolution, context-aware and multi-scale structures), we select three representative nets as substitute models: CSRNet \cite{li2018csrnet} (dilated convolution structure), CAN \cite{liu2019context} (context-aware structure) and DA-Net \cite{zou2018net} (multi-scale structure). We jointly optimize Eq. \ref{equ:9} with substitute models to make the generated black-box patch consist of contextual information and be generally robust towards different model structures. The targeted models are MCNN \cite{zhang2016single} and CMTL \cite{sindagi2017cnn}, which are also the representative models with multi-column structures.\par

The black-box APAM attack results are summarized in the lower part of Table \ref{tab: all attack results}. Compared with the white-box APAM attacks, black-box APAM attacks are somewhat weaker. But we still find some intriguing phenomena: consistent with white-box attacks, CSRNet \cite{li2018csrnet} and CAN \cite{liu2019context} are still vulnerable to the adversarial patch and their performances are severely degraded. Simultaneously, DA-Net \cite{zou2018net} is relatively robust against adversarial patch both in white-box and black-box settings. Particularly, we find DA-Net \cite{zou2018net}, MCNN \cite{zhang2016single} and CMTL \cite{sindagi2017cnn} stay relatively robust when the patch size is small (e.g., $20 \times 20$ and $40 \times 40$). One possible explanation for this is that when training a black-box adversarial patch, it equals implementing some kind of adversarial retraining, which strengthens the involved models.\par

\subsection{Physical Attacks in a Real-world Scenario}
\textbf{Evaluation.} We define the error rate $\pi$ to evaluate the performances of crowd counting models. In the equation, $y$ denotes normal output prediction number and $\tilde{y}$ denotes adversarial prediction number.
\begin{equation}
    Error\ rate:\ \pi = \left| \frac{(y - \tilde{y})}{y} \right| \times 100\%.
    \label{equ:error rate}
\end{equation}

We first choose a well-trained adversarial patch and print it out. Then, we select a large shopping mall as our physical scene and take a group of photos to study (seen in Fig. \ref{fig:physical_attack}). In Fig. \ref{fig:physical_attack}, our proposed attack reaches the error rate $\pi$ as 526.2\% for CSRNet \cite{li2018csrnet}, 982.9\% for CAN \cite{liu2019context}, 252.2\% for DA-Net \cite{zou2018net}, 100\% for MCNN \cite{zhang2016single} and 76.9\% for CMTL \cite{sindagi2017cnn}. From Fig. \ref{fig:physical_attack}, the physical APAM attack does degrade the prediction of networks even the adversarial patch area is less than 6\% of that of the image. We observe that there are two kinds of the network estimation: an extremely large number and a small number. For instances, Fig. \ref{fig:physical_attack} shows that CSRNet \cite{li2018csrnet}, DA-Net \cite{zou2018net} and CAN\cite{liu2019context} incline to predict larger number of the crowd while MCNN \cite{zhang2016single} and CMTL \cite{sindagi2017cnn} tend to estimate the number as approximately zero. Actually, the attack target is set as $l^{\ast}=10\times GT$. We attribute the phenomenon to the physical environment factors (e.g., light, size and location). And it does not mean physically attacking failure since we define the error rate $\pi$ in Eq. \ref{equ:error rate} to consider the model to be successfully physically attacked when its error rate is high. Last but not least, experimental results demonstrate that the physical adversarial patch can severely degrade the performance of models.\par

\subsection{Results of Certified Defense}
To begin with, we summarize the training parameters in Supplementary Materials for better reproducibility. The final training results of the five crowd counting models via randomized ablation are summarized in Table \ref{tab:defense table}. In general, we find that along with the adversarial robustness enhancement is the loss of some clean accuracy (i.e., the MAE or RMSE values decrease in the adversarial environments while they increase in the clean environments). This phenomenon is well explained in \cite{Acemap.82507011}. There exists a balance between the adversarial robustness and the clean accuracy. Compared Table \ref{tab: all attack results} with Table \ref{tab:defense table}, we find the randomized ablation training method helps the crowd counting models to be more robust against APAM attacks (e.g., the $MAE(81)$ value of MCNN \cite{zhang2016single} decreases 64.39 and the $RMSE(81)$ value of MCNN \cite{zhang2016single} decreases 11.28). For other models, we all observe the decrease trend, which demonstrate the practical effectiveness of the proposed method. Moreover, we also observe that the clean accuracy loss is indeed acceptable In Table \ref{tab:defense table}, $MAE(0)$ value of MCNN \cite{zhang2016single} increases from 110.20 to 117.32, $RMSE(0)$ value of MCNN \cite{zhang2016single} increases from 173.20 to 185.44. Compared with the popular used \emph{adversarial training} \cite{goodfellow2014explaining} in Table \ref{tab:ablation study results.}, the $MAE(0)$ and $RMSE(0)$ values of MCNN \cite{zhang2016single} are 121.26 and 191.81, respectively. Adversarial training results of clean examples are worse than those of randomized ablation, which demonstrates the effectiveness. \par

% physical attack
\begin{figure}[t]
\centering
 \includegraphics[width=0.48\textwidth]{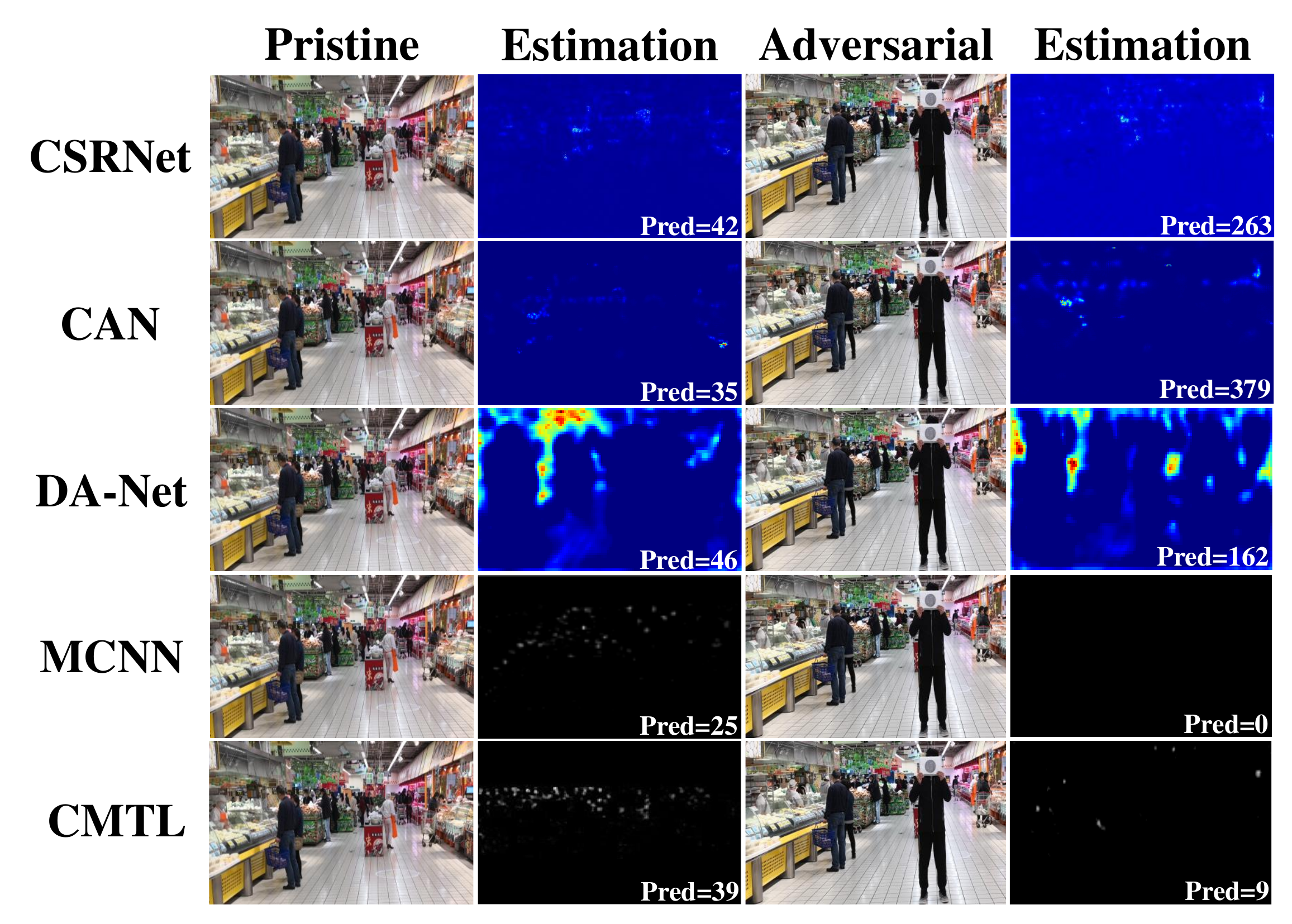}
 \caption{Real-world attack in a physical scenario. We print the patch out and put it in a shopping mall. Surprisingly, after we feed the image to crowd counting networks the prediction of networks are strongly perturbed.}\label{fig:physical_attack}
 \vspace{-7mm}
\end{figure}

\subsection{Ablation Study}
The most intuitive approach to enhance the adversarial robustness of DNNs is \emph{Adversarial Training} \cite{goodfellow2014explaining, Acemap.245763216}. In the adversarial patch scenario, we generate the white-box adversarial patches during the model training loop to enhance the robustness. Note that there are other improved defense methods based on RobustBench\footnote{\url{https://github.com/RobustBench/robustbench}}, their effectiveness on regression models remains an open problem. The experimental results are summarized in Table \ref{tab:ablation study results.}. We now formulate the loss function used in the adversarial training loop:
\begin{equation}
\begin{aligned}
    loss = \lambda \times clean\_loss + (1-\lambda) \times adv\_loss,
\end{aligned}
\end{equation}
where $x_i$ denotes the clean sample images and $l_i$ is the corresponding groundtruth map. $clean\_loss = \frac{1}{2N} \sum^N_{i=1} \vert \vert f (x_i) - l_i \vert \vert^2$ and $adv\_loss = \frac{1}{2N} \sum^N_{i=1} \vert \vert f (\tilde{x}_i) - l_i \vert \vert^2$. $\tilde{x}_i$ is the adversarial example. During the adversarial training loop, we use the following steps:
\begin{itemize}
    \item warm up the model with clean examples in first $m$ epochs
    \item slowly decrease $\lambda$ from 1 to 0.5 in $n$ epochs
    \item maintain the balance $\lambda=0.5$ to finish rest epochs
\end{itemize}
From Table \ref{tab:ablation study results.}, we compare the adversarial training method with randomized ablation ($k=45,100,200$). In conclude, the two methods achieve comparable performances on clean examples (when $k=45,100$ randomized ablation clean MAE value is 3 lower than that of adversarial training). When compared on the white-box APAM attack evaluations, we find the randomized ablation approach beats the adversarial training. For examples, on the small adversarial patch such as size $20\times 20$, the $MAE(20)$ of $k=100$ is 37.13 lower than that of adversarial training. Besides, for large patches such as $163\times163$, the $MAE(163)$ of $k=100$ is 27.85 lower than that of adversarial training.\par

\textbf{Theoretical Possibility of Adversarial Patch via Randomized Ablation.} We use the Theorem \ref{theorem1} to predict the practical possibility of retaining pixels from non-adversarial regions (upper bound possibility $\overline{Pr}$) and adversarial regions (lower bound possibility $\underline{Pr}$). The results of four patches ($20\times 20, 40\times 40, 81\times 81 and 163\times 163$) are summarized in Table \ref{tab:theoretical possibility}. We find that $\overline{Pr}$ decreases as the patch size increases. Specifically, the $\overline{Pr}$ drops sharply (0.6611 to 0.1827) when patch size increases from $81 \times 81$ to $163 \times 163$. Meanwhile, although the lower bound possibility $\underline{Pr}$ increases dramatically (from $1.7\times 10^{-147}$ to $3.9\times 10^{-65}$) when patch size increases, $\underline{Pr}$ is still close to zero and this phenomenon indicates the \textbf{worst case} (depicted in Eq. \ref{proof:worst case}) is unlikely to happen, which guarantees the stability of the proposed randomized ablation method.\par

\section{Conclusion and Discussion}
In this article, we introduce an adversarial patch attack framework named APAM, which poses a severe threat to crowd counting models. We further propose a general defense method to certify the robustness of crowd counting models via randomized ablation. We theoretically and experimentally demonstrate the effectiveness of the proposed method.\par 

\textbf{Physical Evaluations.} Despite the effectiveness of the framework, we now present some limitations. We mainly design our patch in the digital space, and therefore, we indeed find our physical evaluations somewhat weak since elaborately design the physical patch for attack is another story (i.e., light conditions, angles of the patch, human impact and so on). But the experimental results are quite exciting and motivates the further research on the physical evaluations of adversarial robustness of regression models. \par

\textbf{Experiments.} We evaluate the attack and defense framework mainly (indeed, our method is general to other popular datasets such as UCF-CC-50 \cite{idrees2013multi} and UCF-QNRF \cite{idrees2018composition}) on the ShanghaiTech dataset \cite{zhang2016single}. The reason that we only evaluate on one dataset are as follows: \textbf{1)} ShanghaiTech dataset is one of the most representative and challenging datasets in crowd counting \cite{zhang2016single}, as is detailed claimed in Section \ref{sec:experiment_setup}. \textbf{2)} For the real-life adversary, successfully attacking victim models with various structures is definitely more significant. Since we miss the attack baseline to compare, we have added the random patch experiments in Supplementary Materials for the better understanding.\par %Meanwhile, we adopt the adversarial patch training scheme as the defense comparison baselines.\par

\begin{acks}
This work is supported by National Natural Science Foundation of China (NSFC) under grant no. 61972448. (Corresponding author: Pan Zhou). We thank anonymous reviewers for the constructive feedbacks. We thank Xiaodong Wu, Hong Wu and Haiyang Jiang for the helps in physical experiments. We thank Shengqi Chen, Chengmurong Ding, Kexin Zhang and Chencong Ren for their valuable discussions with the work. \par
\end{acks}

\clearpage
\vfill\eject
\bibliographystyle{ACM-Reference-Format}
\bibliography{ref}

\end{document}